\documentclass[11pt,a4paper,reqno]{amsart}
\usepackage[latin1]{inputenc}
\usepackage{mathtools}
\usepackage{amstext}
\usepackage{amsthm}
\usepackage{amssymb}
\usepackage{esint}

\makeatletter

\numberwithin{equation}{section}
\numberwithin{figure}{section}
\theoremstyle{plain}
\newtheorem{thm}{\protect\theoremname}
\theoremstyle{remark}
\newtheorem{rem}[thm]{\protect\remarkname}
\theoremstyle{plain}
\newtheorem{cor}[thm]{\protect\corollaryname}

\usepackage{amsthm}
\usepackage{amsfonts}\usepackage{amsxtra}\usepackage{appendix}\usepackage{euscript}\usepackage{delarray}\usepackage{dsfont}\usepackage{bm}\usepackage{mathrsfs}\usepackage{color}
\usepackage{tikz}
\usepackage{amscd}\usepackage{mathrsfs}\usepackage{bbm}
\usepackage{tabularx}
\usepackage[parfill]{parskip}
 
\usepackage{booktabs}
\usepackage{array}
\usepackage{paralist}
\usepackage{verbatim}
\usepackage{mathrsfs}
\usepackage{amsthm}
\usepackage{amsfonts}\usepackage{graphics}
\usepackage{color}\usepackage{enumerate}

\theoremstyle{plain}

\theoremstyle{remark}
\usepackage{caption}

\renewcommand{\phi}{\varphi}

\renewcommand{\epsilon}{\varepsilon}

\renewcommand{\geq}{\geqslant}
\renewcommand{\leq}{\leqslant}
\renewcommand{\hat}{\widehat}
\renewcommand{\tilde}{\widetilde}

\allowdisplaybreaks

\makeatother

\providecommand{\corollaryname}{Corollary}
\providecommand{\remarkname}{Remark}
\providecommand{\theoremname}{Theorem}

\begin{document}
\title{ Machine Learning and Control Theory }
\author{Alain Bensoussan}
\address{International Center for Decision and Risk Analysis, Jindal School
of Management, The University of Texas at Dallas, TX75080, USA, and
School of Data Science, City University Hong Kong}
\email{alain.bensoussan@utdallas.edu, abensous@cityu.edu.hk}
\author{Yiqun Li}
\address{Department of Statistics, The Chinese University of Hong Kong, Hongkong.}
\email{yiqunli1991@gmail.com}
\author{Dinh Phan Cao Nguyen}
\address{Department of Mathematics, Southern Methodist University, Dallas,
TX 75275, USA and Faculty of Information Technology, Nha Trang University,
Vietnam}
\email{dpcnguyen5690@gmail.com}
\author{Minh-Binh Tran}
\address{Department of Mathematics, Southern Methodist University, Dallas,
TX 75275, USA}
\email{minhbinht@mail.smu.edu}
\author{Sheung Chi Phillip Yam}
\address{Department of Statistics, The Chinese University of Hong Kong, Hong
Kong}
\email{scpyam@sta.cuhk.edu.hk}
\author{Xiang Zhou}
\address{School of Data Science and Department of Mathematics, City University
of Hong Kong, Hong Kong}
\email{xizhou@cityu.edu.hk}

\begin{abstract}
We survey in this article the connections between Machine Learning
and Control Theory. Control Theory provide useful concepts and tools
for Machine Learning. Conversely Machine Learning can be used to solve
large control problems. In the first part of the paper, we develop
the connections between reinforcement learning and Markov Decision
Processes, which are discrete time control problems. In the second
part, we review the concept of supervised learning and the relation
with static optimization. Deep learning which extends supervised learning,
can be viewed as a control problem. In the third part, we present
the links between stochastic gradient descent and mean field theory.
Conversely, in the fourth and fifth parts, we review machine learning
approaches to stochastic control problems,and focus on the deterministic
case, to explain, more easily, the numerical algorithms. 
\end{abstract}

\maketitle

\section{Introduction}

The Big Data phenomenon is at the origin of a new expansion of Artificial
Intelligence. Machine learning \cite{Jordan2015Science} is a way
to implement AI, by providing the machine with the capability of learning
and decision making, which characterize humans. The fact that humans
use algorithms to help performing these two tasks is not new, by itself.
As soon as computing possibilities appeared, algorithms have been
developed. The ambition of AI came also early. However, during the
last decades, the momentum has been spectacular, and Machine Learning
has become the new Graal. Its introduction has revolutionized all
kinds of fields in science, in engineering, in medicine, in management.
Image processing, pattern recognition, text mining, speech recognition,
automatic translation have benefited considerably from this development.
An important breakthough occurred with the methodology of deep neural
network (DNN). 

Conceptually, since the objective is to improve the knowledge of environment
and improve decision making, we are naturally dealing with optimization
and statistics. This is clearly apparent in supervised learning.
On the other hand, reinforcement learning and DNN add an additional
variable, which is time or ordered like time. Control theory comes
in as the framework of dynamic optimization.

Control theory, see for instance, \cite{bensoussan2018book}, is about
how to design optimal actions for dynamical models, in continuous
or discrete time. However, it is notoriously acknowledged that the
numerical computation is the main barrier of putting these control
theories to work in practice and many applications are unfortunately
limited to the linear quadratic regulator. The curse of dimensionality
as Bellman, the creator of Dynamic Programming, coined it has been
haunting the numerical methods of control theory for quite a long
time. It is therefore natural that the new possibilities of ML be
considered to overcome the challenge of dimension. This explains why,
in the past few years, we have been witnessing many exciting ideas
and innovative results from the perspective of merging the above two
research areas, with the efforts from different communities like applied
and computational mathematics, optimal control, stochastic optimization
as well as computer science. The two sides, researchers from machine
learning and optimal control, start to explore the techniques, tools
as well as problem formulations, from each other. We can roughly divide
these works into two categories: \textit{control theory for machine
learning} and \textit{machine learning for control theory}. Generally
speaking, the former refers to the use of control theory as a mathematical
tool to formulate and solve theoretical and practical problems in
machine learning, such as optimal parameter tuning, training neural
network; while the latter is how to use machine learning practice
such as kernel method and DNN to numerically solve complex models
in control theory which can become intractable by traditional methods
(\cite{Han2018PNAS}).

There are many evidences to support our argument of close connections
between machine learning and control theory. We begin with reinforcement
learning (RL), which became famous when AlphaGo Zero \cite{alphagozero2017}
was invented. Reinforcement learning \cite{sutton2018reinforcement}
is a subfield of machine learning that studies how to use past data
to enhance the future manipulation of a dynamical system. The control
communities target for the same problems as RL. However, the RL and
control communities are practically disjoint due to the distinctive
language and culture; see \cite{Recht2019tour} for a recent effort
to unify this gap.

In RL, one of the simplest strategies is to first estimate such models
from the given data, which is called system identification in control
community. This can be achieved by supervised learning \cite{Chiuso2019}.
Then in the second stage dynamical programming principle in control
theory can be applied and to derive many popular RL algorithms such
as Q-learning and Temporal Difference algorithms \cite{sutton2018reinforcement}.

As said above, Dynamic Programming is hard to implement numerically,
for high dimensional dynamic systems. Machine learning and DNN can
be helpful. For example, \cite{Han2018PNAS} proposed an efficient
machine learning algorithm by using DNN to approximate the value function
in the high dimensional Hamilton-Jacobi-Bellman equation, based on
the equivalent stochastic control formulation of the PDE.

The bond that ties machine learning and control theory more closely
in recent years gets critically strengthened from continuous perspective
in various contexts \cite{weinan2017proposal,E2019MLcontinuous,Recht2019tour}.
For example, deep residual neural network (ResNet) \cite{HeResNet}
can be obtained by recasting it as dynamical systems with network
layers considered as time discretization (\cite{chang2017reversible,chang2017multi,chen2018neural,haber2017stable,li2018optimal,li2017deep,sonoda2017double}).
Based on this point of view, machine learning algorithm for ResNet
can be viewed as part of static and dynamic optimization for an ordinary
differential equation controlled by network parameters \cite{E2019RMS}.
This continuous model immediately triggered several new training methods
based on well-known techniques in control theory: \cite{li2017maximum}
from the Pontryagin Maximum Principle and \cite{chen2018neural} from
the adjoint approach. This viewpoint of continuous modelling is also
becoming more and more popular in optimization community for machine
learning, particularly for the stochastic gradient descent(SGD), in
which a stochastic differential equation (SDE) emerges as the continuous
model \cite{li2019stochastic}. The acceleration of the SGD is then
regarded as an optimal control problem for the SDE to reach minimum
point as early as possible \cite{li2017stochastic}. The contribution
of control theory is certainly not restricted to the training algorithm.
For RL, the trade-off between exploration and exploitation is a very
serious and daunting practical problem. \cite{wang2018exploration}
recently studied the analysis of this problem in theory through the
lens of stochastic control. Similarly, the need to provide a solid
mathematical framework to analyze deep neural networks is very pressing.
Recent works have pointed out that new mathematical properties of
deep neural network can be obtained by recasting deep learning as
dynamical systems (cf. \cite{chang2017reversible,chang2017multi,chen2018neural,haber2017stable,li2017maximum,li2018optimal,li2017deep,lu2017beyond,sonoda2017double}).

Nowadays, it is difficult to ignore the intervene and synthesis between
machine learning and control theory and the fusion of these two fields
at certain boundaries is pushing forward tremendous research progress
with accelerating momentum now. This paper is to give a brief introduction
to and a short review of some selective works on the overlap of these
communities. The interaction between the data-driven approach in machine
learning and the model-based control theory is still at the very early
age and there are certainly many challenges at the control-learning
interface to advance the deeper development both in theory and in
practice. We hope that the gap between the learning-centric views
of ML and the model-centric views of control can diminish in the foreseen
future on an arduous journey of understanding machine learning and
artificial intelligence. As a result, a new territory may emerge (e.g.
\textit{actional intelligence} in \cite{Recht2019tour}) from these
joint efforts across the disciplines.

In the first part of the article, we discuss Markov Decision Processes
(MDP), which normally provide mathematical frameworks for modeling
decision making in stochastic environment where outcomes are partly
random and partly under the control of a decision maker. MDPs can
indeed be solved via Dynamic Programming and provide a very useful
framework for Reinforcement Learning.

The second part of the article is devoted to Supervised Learning and
Deep Learning, that concerns the approximation of a function given
some preliminary observations. Supervised Learning is an optimization
problem. Deep Learning can be recast into a control theory problem
and can be solved using various strategies, including the Pontryagin
Maximum Principle approach.

Recent mean field and stochastic control views for Stochastic Gradient
Descent methods will be provided in the third part of this paper.

In the next section, we study a Stochastic Control Problem, in which
the state is that of a controlled diffusion. We then propose a Machine
Learning approach for this problem.

We finally focus on the deterministic case in section 6 to simplify
the theory. We provide some related theoretical results in companion
with a few high-dimensional numerical illustrations to demonstrate
the effectiveness of the algorithms.

\section{Reinforcement learning}

\subsection{General concepts}

In the language of Control Theory, we consider a dynamical system,
which evolves in an uncertain environment. The evolution of this system
is called a process, which can be characterized by its state. A controller
decides a strategy of actions, called feedback, and there is at each
time a cost or profit attached to the current state and the currrent
action. In the language of RL, every time the action is made, the
controller receives an award. The controller will try to choose the
actions such that the sum of rewards is maximized. Since time is discrete,
the control problem is called a Markov Decision Process (MDP) and
can be solved by Dynamic Programming approach. The award is then a
function of the state and the action.

\subsection{Mathematical model without action}

We suppose that the states belong to a space $X$. On $X$, there
is a $\sigma$-algebra, denoted by $\mathcal{X}$. A transition probability
is a (regular enough) function $\pi(x;\Gamma)$ on $(X,\mathcal{X})$.
For any fixed $x$ we define the probability of $\Gamma\in\mathcal{X}$
to be $\pi(x;\Gamma)$. If $B$ is the space of bounded functions
on $X,$ equipped with the norm $\|f\|=\sup_{x}|f(x)|$, we associate
to the transition probability a linear operator $\Phi$ from $B$
to $B$ as follows:

\begin{equation}
\Phi f(x)=\int_{X}f(\eta)\pi(x;\mathrm{d}\eta),\label{eq:1-1}
\end{equation}
and clearly $||\Phi||\leq1.$ A Markov chain $\{X_{i}\}_{i=1}^{\infty}$
on $X$ with transition probability $\pi(x;\mathrm{d}\eta)$ is a
stochastic process on $X$ such that 
\begin{equation}
\mathbb{E}[f(X_{n+1})|X_{n}=x]=\int_{X}f(\eta)\pi(x;\mathrm{d}\eta),\text{ for }n=1,2,\cdots.\label{eq:1-2}
\end{equation}
Assuming stationarity in \eqref{eq:1-2}, and choosing $\alpha$ to
be a discount factor, we can express the function

\begin{equation}
u(x)=\mathbb{E}\left[\sum_{n=1}^{+\infty}\alpha^{n-1}f(X_{n})\Bigg|X_{1}=x\right].\label{eq:1-3}
\end{equation}
This is the sum of rewards, in the terminology of RL. There is no
action to modify the trajectory. We just add the discounted rewards.
We can give an explicit analytic expression ( not probabilist) of
the function $u(x)$. It is the unique solution of the analytic problem
\begin{equation}
u=f+\alpha\Phi u.\label{eq:1-4}
\end{equation}
It then follows that 
\begin{equation}
u=(I-\alpha\Phi)^{-1}f,\label{eq:1-5}
\end{equation}
and using the generator $\Phi$, we can also rewrite, 
\begin{equation}
u=\sum_{n=1}^{\infty}\alpha^{n-1}\Phi^{n-1}f.\label{eq:1-6}
\end{equation}

\subsection{Approximation}

Our main task now is to compute the function $u(x).$ The equations
\eqref{eq:1-4} and \eqref{eq:1-6} are explicit and straightforward.
However, the challenge is when the dimension $d$ of $X$ is large,
these formulations are not of practical use. Here comes the other
aspect of machine learning: how to approximate a function given by
formulae \eqref{eq:1-3} or \eqref{eq:1-4}. Since Supervised Learning
does exactly that, approximate a function, we follow the ideas of
SL. There are basically two methods, the \textit{parametric} method
and the \textit{non parametric }method. In the parametric method,
we look for an approximation of the form

\begin{equation}
u(x)\approx\sum_{i=1}^{I}\theta_{i}\varphi_{i}(x),\label{eq:2-1}
\end{equation}
where $\varphi_{i}(x)$ are given functions, so that the family $\{\varphi_{i}(\cdot)\}_{i=1}^{\infty}$
forms a basis of the functional space to which $u(x)$ belongs, and
$\theta_{i}$ are coefficients to be determined. We need to compute
parameters minimizing the error

\begin{equation}
\left\Vert \sum_{i=1}^{I}\theta_{i}(\varphi_{i}-\alpha\Phi\varphi_{i})-f\right\Vert ^{2},\label{eq:2-2}
\end{equation}
where $\|\cdot\|$ is the sup-norm $\|G\|=\sup|G|$. To guarantee
the existence and uniqueness of the parameters, we minimize the quadratic
functional, with a quadratic regularization.

\begin{equation}
\gamma\sum_{i=1}^{I}\theta_{i}^{2}+\left\Vert \sum_{i=1}^{I}\theta_{i}(\varphi_{i}-\alpha\Phi\varphi_{i})-f\right\Vert ^{2}.\label{eq:2-3}
\end{equation}
In the non parametric method, used in supervised learning, we do
not refer to a functional equation for $u(x).$ We assume that we
can compute the value at a finite number of points. For a given $x$,
$u(x)$ can be calculated by formula \eqref{eq:1-3}, by Monte Carlo
simulation. We then find 
\begin{equation}
u(x)\approx\dfrac{1}{N}\sum_{\nu=1}^{N}\sum_{n=1}^{+\infty}\alpha^{n-1}f(X_{n}^{\nu}),\label{eq:2-4}
\end{equation}
where $X_{1}^{\nu}=x,\cdots,X_{n}^{\nu}=X_{n}(\omega^{\nu}),\cdots,$
represents one trajectory indexed by $\nu$ of the Markov chain, corresponding
to one sample point $\omega^{\nu}$. We choose $M$ points $x^{1},\cdots,x^{M}$
in $\mathbb{R}^{d}$, and then compute $u(x^{1})=y^{1},\cdots u(x^{M})=y^{M}$
by using Monte-Carlo method and the approximation formula \eqref{eq:2-4}.
The number $M$ is chosen arbitrarily. If we assume that $f$ is continuous
and bounded, then $u(x)$ is also bounded and continuous. The goal
is to extrapolate $u(x)$ from the knowledge of $y^{1},\cdots,y^{M}.$

We now choose a subset $\mathcal{H}$ of $C(\mathbb{R}^{d})$. This
subset is called the hypothesis space. We select an element in $\mathcal{H}$
such that it is the closest possible to $y^{1},\cdots,y^{M}$ at points
$x^{1},\cdots,x^{M}.$ We assume naturally that $\mathcal{H}$ is
a nice enough functional space. The theory of reproducing kernels
allows us to define $\mathcal{H}$ as a Hilbert space, with a continuous
injection in $C(\mathbb{R}^{d}$). The function $u(x)$ can be defined
as the solution of the minimization problem 
\begin{equation}
\min_{u\in\mathcal{H}}\left\{ \gamma\|u\|_{\mathcal{H}}^{2}+\sum_{m=1}^{M}(u(x^{m})-y^{m})^{2}\right\}.\label{eq:2-5}
\end{equation}

\begin{rem}
\label{rem2-1}In RL, one claims that a significant difference with
MDP is that the Markov Chain may not be known. The controller, however
makes trials, which is similar to MonteCarlo, whithout referring to
a selection of trajectories according to a given probability transition. 
\end{rem}

\subsection{Mathematical model with action}

The Markov chain has a probability transition depending on an auxiliary
variable $a$ called the action, $\pi(x,a;\:\mathrm{d}\eta)$. When
the action is a function of the state, also called the feedback, $a(x),$
we then get $\pi(x,a(x);\mathrm{d}\eta)$. We now define the operator
$\Phi^{a}f(x)$ or $\Phi^{a(x)}f(x)$ by 
\begin{equation}
\Phi^{a}f(x)=\int_{X}f(\eta)\pi(x,a;\mathrm{d}\eta);\quad\Phi^{a(x)}f(x)=\int_{X}f(\eta)\pi(x,a(x);\mathrm{d}\eta).\label{eq:2-6}
\end{equation}
We consider an award depending on the state and the action $f(x,a)$.
For convenience, this award is supposed to be a cost and we assume
$f(x,a)\geq0.$ We then set the aggregate cost to be 
\begin{equation}
J_{a(\cdot)}(x)=\mathbb{E}\left[\sum_{n=1}^{+\infty}\alpha^{n-1}f(X_{n},a(X_{n}))\Bigg|X_{1}=x\right],\label{eq:2-7}
\end{equation}
in which $X_{n}$ evolves as a function of the probability transition
$\pi(x,a(x);\mathrm{d}\eta).$ We also define the value function 
\begin{equation}
u(x)=\inf_{a(\cdot)}J_{a(\cdot)}(x),\label{eq:2-8}
\end{equation}
which is the solution of the following Bellman equation 
\begin{equation}
u(x)=\inf_{a}[f(x,a)+\alpha\Phi^{a}u(x)].\label{eq:2-9}
\end{equation}
It is also interesting to introduce the cost ($Q$-function) when
the first action is arbitrary and the following actions are optimized,
namely 
\begin{equation}
Q(x,a)=f(x,a)+\alpha\Phi^{a}u(x).\label{eq:2-10}
\end{equation}
Taking into account the fact that 
\begin{equation}
u(x)=\inf_{a}Q(x,a),\label{eq:2-11}
\end{equation}
we arrive at 
\begin{equation}
Q(x,a)=f(x,a)+\alpha[\Phi^{a}(\inf_{a'}Q(\cdot,a'))](x).\label{eq:2-12}
\end{equation}
There are two basic types of iteration to solve the above Bellman
equation, the \textit{value iteration} and the \textit{policy iteration}.
The \textit{value iteration} is defined by 
\begin{equation}
u_{k+1}(x)=\inf_{a}[f(x,a)+\alpha\Phi^{a}u_{k}(x)],\label{eq:2-13}
\end{equation}
and $u_{0}(x)=0.$ When $f(x,a)$ is bounded, the solution of the
Bellman equation is unique and the sequence $u_{k}(x)$ converges
to the the value function monotonically. When $f(x,a)$ is not bounded,
the solution of the Bellman equation is not unique. The sequence $u_{k}(x)$
converges monotonically to the value function, which is the minimum
solution. We can also interpret $u_{k}(x)$ as the value function
for the control problem with $k$ periods. To see this, we define
\begin{equation}
J_{a(\cdot)}^{k}(x)=\mathbb{E}\left[\sum_{n=1}^{k}\alpha^{n-1}f(X_{n},a(X_{n}))\Bigg|X_{1}=x\right],\label{eq:2-14}
\end{equation}
then 
\begin{equation}
u_{k}(x)=\inf_{a(\cdot)}J_{a(\cdot)}^{k}(x).\label{eq:2-15}
\end{equation}
On the other hand, the \textit{policy iteration} technique starts
with a given feedback control $a^{k}(x)$ and solves the linear (fixed
point) problem similar to \eqref{eq:1-4} 
\begin{equation}
u^{k+1}(x)=f(x,a^{k}(x))+\alpha\Phi^{a^{k}(x)}u^{k+1}(x).\label{eq:2-16}
\end{equation}
With the $u^{k+1}(x)$, then $a^{k+1}(x)$ is defined by 
\[
\inf_{a}[f(x,a)+\alpha\Phi^{a}u^{k+1}(x)],
\]
in which we start with a function $a^{0}(x)$, which minimizes 
\[
\inf_{a}f(x,a).
\]
In both types of iterations, we have to solve the respective minimization
problems\\
 (i) value iteration: $\quad\quad\quad\inf_{a}[f(x,a)+\alpha\Phi^{a}u_{k}(x)];$\\
 (ii) policy iteration: $\quad\quad\quad\inf_{a}[f(x,a)+\alpha\Phi^{a}u^{k+1}(x)],$\\
 in which $u_{k}(x)$ and $u^{k+1}(x)$ are known respectively (obtained
a-priori by some approximation method such as the use of reproducing
kernals).

For both cases, the minimization problem can be resolved by a gradient
decent technique. We set 
\[
Q^{k+1}(x,a)=f(x,a)+\alpha\Phi^{a}u^{k+1}(x)
\]
which is an approximation of $Q(x,a).$

Since we cannot find the infimum exactly, we use the following approximation
for $a^{k+1}(x)$ 
\begin{equation}
a^{k+1}(x)=a^{k}(x)-\rho^{k}D_{a}Q^{k+1}(x,a^{k}(x)),\label{eq:2-17}
\end{equation}
in which the coefficient $\rho^{k}$ is chosen such that it solves
the following scalar optimization problem 
\begin{equation}
\inf_{\rho}Q^{k+1}(x,a^{k}(x)-\rho D_{a}Q^{k+1}(x,a^{k}(x))).\label{eq:2-18}
\end{equation}
By definition, 
\[
u^{k+1}(x)=Q^{k+1}(x,a^{k}(x)),
\]
We can then, instead of solving the linear problem for $u^{k+1}(x)$,
use the apprximation $Q^{k}(x,a^{k}(x)).$ We then proceed as follows:
knowing $a^{k}(x)$ and $Q^{k}(x,a)$, we define 
\begin{equation}
\begin{cases}
\bar{u}^{k+1}(x):=Q^{k}(x,a^{k}(x)),\label{eq:2-19}\\
Q^{k+1}(x,a)=f(x,a)+\alpha\Phi^{a}\bar{u}^{k+1}(x),
\end{cases}
\end{equation}
and $a^{k+1}(x)$ can then be obtained approximately through \eqref{eq:2-17}
and \eqref{eq:2-18}. In the above procedure, we start with $Q^{0}(x,a)=f(x,a)$
and $a^{0}(x)$ is chosen to be the minimizer of $f(x,a).$

\section{Control theory and deep learning}

\subsection{Supervised learning}

Supervised learning concerns basically approximating an unknown function
$F(x):\:\mathbb{R}^{d}\rightarrow\mathbb{R},$ given some noisy observations
$y^{m}=F(x^{m})+\epsilon^{m},$ in which $x^{m}$ is known and the
noise $\epsilon^{m}$ models the uncertainty (unknown). There are
two methods that can be used to solve this approximation problem.
In the first method, we consider a function $f(x;\theta),$ where
$\theta\in\mathbb{R}^{n}$ for some $n$ and we try to perform a minimization
problem with the parameter $\theta$ 
\begin{equation}
\min_{\theta}\left\{ \gamma|\theta|^{2}+\sum_{m=1}^{M}(f(x^{m};\theta)-y^{m})^{2}\right\}.\label{eq:3-1}
\end{equation}
This is the well-known \textit{parametric method}. In the simplest
case of neural networks, the function $f(x;\theta)$ is defined as
follows:

We first introduce $X=\chi(x),$ where $\chi:\mathbb{R}^{d}\rightarrow\mathbb{R}^{n}$.
Of course $\chi$ can be identity. Now, choose $\sigma$ to be a
scalar function, called the \textit{activation function} and $W$
to be a matrix in $\mathcal{L}(\mathbb{R}^{n};\mathbb{R}^{d})$. Suppose
that $b$ is a vector in $\mathbb{R}^{n}.$ The pair $(W,b)$ represents
the parameter $\theta.$ We now define the vector $\tilde{X}$ by
\begin{equation}
\tilde{X}_{i}=\sigma\left(\sum_{j=1}^{d}W_{ij}X_{j}+b_{i}\right),\label{eq:3-2}
\end{equation}
and 
\begin{equation}
f(x;\theta)=g(\tilde{X}),\label{eq:3-3}
\end{equation}
where $g:\mathbb{R}^{n}\rightarrow\mathbb{R}$ is the output function.
The minimization in \eqref{eq:3-1} is performed by a gradient method
with iterative application of chain rules.

In the \textit{non-parametric method}, particularly, the \textit{kernel
method}, one finds a functional space $\mathcal{H}$, to which the
approximation of $F(x)$ belongs. One then chooses that approximation
$f(x)$ by solving 
\begin{equation}
\min_{f\in\mathcal{H}}\left\{ \gamma\|f\|_{\mathcal{H}}^{2}+\sum_{m=1}^{M}(f(x^{m})-y^{m})^{2}\right\}.\label{eq:3-4}
\end{equation}

\subsection{Deep learning}

Deep learning is a generalization of supervised learning with a sequence
of layers. We generalize \eqref{eq:3-3} to $K$ layers as follows
with $X^{k}\in\mathbb{R}^{n},\theta^{k}:=(W^{(k+1)},b^{(k+1)})\in\mathcal{L}(\mathbb{R}^{n};\mathbb{R}^{n})\times\mathbb{R}^{n},$
\begin{equation}
X^{k+1}=f_{k}(X^{k};\theta^{k}),\:k=0,\cdots K-1,\label{eq:3-5}
\end{equation}
where 
\begin{align*}
f_{k}(X^{k};\theta^{k}):=\sigma(W^{(k+1)}X^{k}+b^{(k+1)})
\end{align*}
and 
\begin{equation}
X^{0}=\chi(x),\label{eq:3-6}
\end{equation}
with 
\begin{equation}
f(x,\theta)=g(X^{K}).\label{eq:3-7}
\end{equation}
The parameter $\theta$ is the collection of $\{\theta^{0},\cdots\theta^{K-1}\}.$
We have written the case of several layers of neural networks, but
it is just an example. 

\subsection{Control theory approach}

This approach has been introduced by Li, Tai, E \cite{li2015dynamics,li2017stochastic,li2019stochastic}.
The idea is to consider a continuous time extension of \eqref{eq:3-5},
\eqref{eq:3-6} and \eqref{eq:3-7}. We write 
\begin{equation}
\begin{cases}
\dfrac{\mathrm{d}X}{\mathrm{d}t}=f(X_{t},\theta_{t},t),\\
X_{0}=\chi(x).
\end{cases}\label{eq:3-8}
\end{equation}
The approximation of $F(x)$ is then $f(x,\theta)=g(X_{T}).$ The
loss in this scenario is $(y-g(X_{T}))^{2}=\Phi(X_{T}),$ recalling
that $y$ will be a known value. The idea is to consider $\theta_{t}$
as a control and we want to minimize an analog of \eqref{eq:3-4},
which is expressed as: 
\begin{equation}
J(\theta):=\sum_{m=1}^{M}\Phi(X_{T}^{m})+\int_{0}^{T}L(\theta_{t})\mathrm{d}t,\label{eq:3-9}
\end{equation}
where $L(\theta)$ is a regularization function, for instance, $\gamma\|\theta\|^{2}$.
Using a Pontryagin Maximum Principle approach, we can write a necessary
condition of optimality for the control $\hat{\theta}_{t}.$ Define
$(\hat{X}_{t}^{m},\hat{p}_{t}^{m})$, we have the optimal state and
optimal adjoint state solutions of 
\begin{equation}
\left\{ \begin{aligned} & \dfrac{\mathrm{d}\hat{X_{t}}^{m}}{\mathrm{d}t}=f(\hat{X}_{t}^{m},\hat{\theta}_{t},t),\;\quad\hat{X}_{0}^{m}=x^{m},\;\\
 & -\dfrac{\mathrm{d}\hat{p}_{t}^{m}}{\mathrm{d}t}=(D_{x}f)^{*}(\hat{X}_{t}^{m},\hat{\theta}_{t},t)\hat{p}_{t}^{m},\quad\hat{p}_{T}^{m}=D_{x}\Phi(\hat{X}_{T}^{m}),
\end{aligned}
\right.\label{eq:3-10}
\end{equation}
and the optimality condition 
\[
\hat{\theta}_{t}\;\text{minimizes }H(\hat{X}_{t},\hat{p}_{t},\theta,t),\:\text{a.e.}\:t,
\]
with 
\[
H(\hat{X}_{t},\hat{p}_{t},\theta,t)=\sum_{m=1}^{M}\hat{p}_{t}^{m}f(\hat{X}_{t}^{m},\theta,t)+L(\theta).
\]
To solve \eqref{eq:3-10}, one can use the following approximation
recursively: let $\hat{\theta}_{t}^{k}$ be given, define $(\hat{X}_{t}^{m,k},\hat{p}_{t}^{m,k})$
by 
\begin{equation}
\left\{ \begin{aligned} & \dfrac{\mathrm{d}\hat{X_{t}}^{m,k}}{\mathrm{d}t}=f(\hat{X}_{t}^{m,k},\hat{\theta}_{t}^{k},t),\;\quad\hat{X}_{0}^{m,k}=x^{m},\;\\
 & -\dfrac{\mathrm{d}\hat{p}_{t}^{m,k}}{\mathrm{d}t}=(D_{x}f)^{*}(\hat{X}_{t}^{m,k},\hat{\theta}_{t}^{k},t)\hat{p}_{t}^{m,k},\quad\hat{p}_{T}^{m,k}=D_{x}\Phi(\hat{X}_{T}^{m,k}).\label{eq:3-11}
\end{aligned}
\right.
\end{equation}
We look for $\hat{\theta}_{t}^{k+1}$ that minimizes 
\begin{equation}
\sum_{m}\hat{p}_{t}^{m,k}f(\hat{X}_{t}^{m,k},\theta,t)+L(\theta).\label{eq:3-12}
\end{equation}
Note that the above approximation may fail to converge. We refer to
\cite{E2019RMS,E2019MLcontinuous,Han2016,sutton2018reinforcement}
for recent improvements and techniques to deal with this issue.

\section{stochastic gradient descent and control theory}

\subsection{Comments}

The gradient descent algorithm plays an essential role in the various
parts of ML, as we have seen above. There has been a considerable
amount of work in order to improve its efficiency. The stochastic
version it, described below, offers another example of connection
to control theory, stochastic control and even mean field control
theory. We limit ourselves to some basic considerations. In particular
we do not discuss the case of the use of SG for DNN, as in \cite{mei2018mean}.
This is because the connection is of a different nature. It does not
involve control theory, but introduces interesting PDE. In this section, 
we will define SG and relate the choice of the optimal descent parameters
to an MDP problem. We will then give a continuous version and also
connect with mean field control. 

\subsection{Stochastic Gradient and MDP }

We recall the defintion of gradient descent. If $f(x)$ is a function
on $\mathbb{R}^{d},$ for which we want to find a minimum $x^{*}.$ The gradient
descent algorithm is defined by the sequence 

\begin{equation}
x_{k+1}=x_{k}-\eta_{k}Df(x_{k}),\label{eq:4-1}
\end{equation}
where $\eta_{k}$ is a positive constant, which can be chosen independent
of $k$. This is simpler, but by all means not optimal. Suppose now
that the function $f(x)$ is an expected value 

\begin{equation}
f(x)=\mathbb{E}(f(x,Z)).\label{eq:4-100}
\end{equation}
Applying the gradient method to this function leads to 

\[
x_{k+1}=x_{k}-\eta_{k}\mathbb{E}(D_{x}f(x_{k},Z)).
\]
In the SG descent method, one chooses a sequence of independent versions
of $Z$, called $Z_{k}$ and define 

\begin{equation}
X_{k+1}=X_{k}-\eta_{k}D_{x}f(X_{k},Z_{k}).\label{eq:4-101}
\end{equation}
We clearly obtain a controlled Markov chain, in which $\eta_{k}$
is the control. If we define the $\sigma$ -algebra $\mathcal{F}^{k}=\sigma(Z_{1},\cdots,Z_{k})$,
then $\eta_{k}$ is adapted to $\mathcal{F}^{k-1},$ $k\geq1.$ $\mathcal{F}^{0}$
is the trivial $\sigma$-algebra. The process $X_{k}$ is also adapted
to $\mathcal{F}^{k-1}.$ We have to define the pay off to optimize.
Suppose we stop at $K.$ We naturally want $X_{K}$ as close as possible
to $x^{*}.$ One way to proceed would be to minimize $\mathbb{E}f(X_{K}).$
But this requires the computation of $f(x)$, for random values
of the argument, which we want to avoid. In fact, if we insure that
$X_{K}$ is close to a constant, that constant will be necessarily
$x^{*},$ provided $\eta_{k}$ is larger than a fixed positive constant.
So a good criterion will be to minimize 

\begin{equation}
\mathbb{E}|X_{K}-EX_{X}|^{2}.\label{eq:4-102}
\end{equation}
This is not a standard MDP, but a Mean Field type control problem
in discrete time.

\subsection{CONTINUOUS VERSION}

We first write (\ref{eq:4-101}) as follows: 

\[
X_{k+1}=X_{k}-\eta_{k}Df(X_{k})+\eta_{k}\widetilde{Y}_{k},
\]
 with 

\[
\widetilde{Y}_{k}=Df(X_{k})-D_{x}f(X_{k},Z_{k}).
\]
 Note that we have $\mathbb{E}[\widetilde{Y}_{k}|\mathcal{F}^{k-1}]=0$ and 

\[
\mathbb{E}[\widetilde{Y}_{k}(\widetilde{Y}_{k})^{*}|\mathcal{F}^{k-1}]=\Sigma(X_{k}),
\]
 with 

\begin{equation}
\Sigma(x)=\mathbb{E}(D_{x}f(x,Z)(D_{x}f(x,Z))^{*})-Df(x)(Df(x))^{*}.\label{eq:4-103}
\end{equation}
We shall write 

\begin{equation}
\Sigma(x)=\sigma(x)\sigma(x)^{*}.\label{eq:4-104}
\end{equation}
If we write $\widetilde{Y}_{k}=\sigma(X_{k})Y_{k}$, then the process
$Y_{k}$ satisfies 

\begin{equation}
\mathbb{E}[Y_{k}|\mathcal{F}^{k-1}]=0,\;\mathbb{E}[Y_{k}(Y_{k})^{*}|\mathcal{F}^{k-1}]=0.\label{eq:4-105}
\end{equation}
 We obtain the algorithm 

\begin{equation}
X_{k+1}=X_{k}-\eta_{k}Df(X_{k})+\eta_{k}\sigma(X_{k})Y_{k}.\label{eq:4-106}
\end{equation}
We can then follow Li, Tai and E \cite{li2017stochastic} to define
a diffusion approximation of (\ref{eq:4-106}) as follows 

\begin{equation}
dX=-u(t)Df(X(t))dt+u(t)\eta\sigma(X(t))dB(t)\label{eq:4-107}
\end{equation}
where $B(t)$ is a brownian motion and $u(t)$ adapted to the filtration
generated by the brownian motion, with values in $[0,1].$ The number
$\eta$ is a scaling constant. We can then choose the control to minimize
the payoff $\mathbb{E}f(X(T))$ which defines a stochastic control problem,
or 

\begin{equation}
\min_{u(\cdot)\geq u_{0}>0}\mathbb{E}\left(|X(T)-\mathbb{E}X(T)|^{2}\right),\label{eq:4-108}
\end{equation}
which defines a mean field type control problem. 

\section{Machine learning approach of stochastic control problems}

\subsection{General theory}

Let us now consider the following problem, in which the state equation
is a controlled diffusion 
\begin{equation}
\begin{cases}
\mathrm{d}x(t)=g(x(t),a)\mathrm{d}t+\mathrm{d}w,\label{eq:4-1}\\
x(0)=x,
\end{cases}
\end{equation}
and the pay-off is given by 
\begin{equation}
J_{x}(a(\cdot))=\mathbb{E}\left[\int_{0}^{+\infty}\exp(-\alpha t)\:f(x(t),a(t))\mathrm{d}t\right].\label{eq:4-2}
\end{equation}
There are two approaches to resolve the above problem: \textit{Dynamic
Programming} and \textit{Stochastic Pontryagin Maximum Principle}.
The theory shows that the optimal control is described by a feedback
. The value function is defined by 
\begin{equation}
u(x)=\inf_{a(\cdot)}J_{x}(a(\cdot)).\label{eq:4-3}
\end{equation}
In the above problem, there are 3 functions of interest, the value
function $u(x),$ the optimal feedback $\hat{a}(x)$ (if exists),
the gradient of $u(x)$, $\lambda(x)=Du(x).$ Introducing $u(x)$
and its gradient independently may look superfluous. It turns out
that the gradient has a very interesting interpretation, the shadow
price in economics. Surprisingly, the gradient is solution of a self-contained
vector equation. On the numerical side, approximating the gradient
of $u(x)$ by the gradient of the approximation of $u(x)$ results
in a source of errors. This justifies the interest in the system of
equations for $\lambda(x).$ We may think of \textit{parametric and
non-parametric approximations} for these functions. We shall discuss
a parametric approach for the optimal feedback, and a non-parametric
approach for the value function and its gradient.

\subsection{Parametric approach for the feedback}

Let us now replace candidacy of $a(x)$ by a special function of the
form $\tilde{a}(x,\theta)$ with $\theta$ being a parameter in $\mathbb{R}^{p}$.
The function $g(x,a)$ is then replaced by $g(x,\tilde{a}(x,\theta))$,
which can be renamed as $g(x,\theta)$ by abuse of notation. We obtain
the control problem 
\begin{equation}
\begin{cases}
\mathrm{d}x(t)=g(x(t),\theta)\mathrm{d}t+\mathrm{d}w,\label{eq:4-4}\\
x(0)=x,
\end{cases}
\end{equation}
\begin{equation}
J_{x}(\theta(\cdot))=\mathbb{E}\left[\int_{0}^{+\infty}\exp(-\alpha t)\:f(x(t),\theta(t))\mathrm{d}t\right],\label{eq:4-5}
\end{equation}
where $f(x,\theta)$ abbreviates for $f(x,\tilde{a}(x,\theta)).$
The important simplification of this procedure is that $\theta(t)$
is regarded as deterministic.

We can write a necessary condition of optimality for the optimal new
control $\hat{\theta}(t)$. Define the optimal state $\hat{x}(t)$
and the adjoint state $\hat{p}(t)$ by: 
\begin{equation}
\begin{cases}
\mathrm{d}\hat{x}=g(\hat{x},\hat{\theta})\mathrm{d}t+\mathrm{d}w,\quad\hat{x}(0)=x,\label{eq:4-6}\\
-\dfrac{\mathrm{d}\hat{p}}{\mathrm{d}t}+\alpha\hat{p}=g_{x}^{*}(\hat{x},\hat{\theta})\hat{p}+f_{x}(\hat{x},\hat{\theta}),
\end{cases}
\end{equation}
and $\hat{\theta}(t)$ satisfies 
\begin{equation}
\inf_{\theta}\mathbb{E}\,[\hat{p}(t)\cdot g(\hat{x}(t),\theta)+f(\hat{x}(t),\theta)],\:t\text{-a.e.}.\label{eq:4-7}
\end{equation}
To obtain $\hat{\theta}(t)$, we can use an iterative approximation
coupled with a gradient method 
\begin{equation}
\begin{cases}
\mathrm{d}\hat{x}^{k}=g(\hat{x}^{k},\hat{\theta}^{k})\mathrm{d}t+\mathrm{d}w,\quad\hat{x}(0)=x,\label{eq:4-8}\\
-\dfrac{d\hat{p}^{k}}{dt}+\alpha\hat{p}^{k}=g_{x}^{*}(\hat{x}^{k},\hat{\theta}^{k})\hat{p}^{k}+f_{x}(\hat{x}^{k},\hat{\theta}^{k}),
\end{cases}
\end{equation}
\[
\hat{\theta}^{k+1}(t)=\hat{\theta}^{k}(t)-\rho^{k}(t)\mathbb{E}[g_{\theta}^{*}(\hat{x}^{k},\hat{\theta}^{k})\hat{p}^{k}+f_{\theta}(\hat{x}^{k},\hat{\theta}^{k})],\:t\text{-a.e.},
\]
where $\rho^{k}(t)$ minimizes in $\rho$ the following function 
\[
\mathbb{E}\,\hat{[p}(t)\cdot g(\hat{x}(t),\theta)+f(\hat{x}(t),\theta)],\:\text{in which }\theta=\hat{\theta}^{k}(t)-\rho\mathbb{E}[g_{\theta}^{*}(\hat{x}^{k},\hat{\theta}^{k})\hat{p}^{k}+f_{\theta}(\hat{x}^{k},\hat{\theta}^{k})].
\]

\subsection{Non-parametric approach for the value function and its gradient}

First, we notice that the value function $u(x),$ the gradient $\lambda(x)=Du(x)$
and the optimal feedback $\hat{a}(x)$ are linked as follows\footnote{\eqref{eq:4-9}$_{1}$ is the HJB system and \eqref{eq:4-9}$_{2}$
follows by differentiating \eqref{eq:4-9}$_{1}$ with respect to
$x$.} 
\begin{equation}
\begin{cases}
\alpha u(x)=\lambda(x)\cdot g(x,\hat{a}(x))+f(x,\hat{a}(x))+\dfrac{1}{2}\text{tr}(D\lambda(x)),\\
\alpha\lambda(x)=D\lambda(x)g(x,\hat{a}(x))+D_{x}^{*}g(x,\hat{a}(x))\lambda(x)+D_{x}f(x,\hat{a}(x))+\dfrac{1}{2}\Delta\lambda(x),
\end{cases}\label{eq:4-9}
\end{equation}
where 
\[
\hat{a}(x)\:\text{minimizes in }a\text{ of the function }\:\lambda(x)\cdot g(x,a)+f(x,a).
\]
The above system has an interesting structure, in which there is coupling
only for the last two equations. The first equation allows one to
define the value function. Note that we have used the fact that $D\lambda(x)=(D\lambda(x))^{*}.$ 

We now define the following iteration: suppose that we know $(\hat{a}^{k}(x),\lambda^{k}(x))$,
we can find $\lambda^{k+1}(x)$ by solving the differential equation
system: 
\begin{equation}
\alpha\lambda^{k+1}(x)-D\lambda^{k+1}(x)g(x,\hat{a}^{k}(x))-\dfrac{1}{2}\Delta\lambda^{k+1}(x)=D_{x}^{*}g(x,\hat{a}^{k}(x))\lambda^{k}(x)+D_{x}f(x,\hat{a}^{k}(x))\label{eq:4-10}
\end{equation}
such that 
\[
\hat{a}^{k+1}(x)\:\text{minimizes in }a\text{ of the function }\lambda^{k+1}(x)\cdot g(x,a)+f(x,a).
\]
The equations for the components of $\lambda^{k+1}(x)$ are completely
decoupled, and can be solved in parallel. One possibility is to use
simulation to define $\lambda^{k+1}(x)$ in a finite number of points
and to use an extrapolation by a kernel method.

\section{Focus on the deterministic case}

In this section, we shall simplify by considering the case of a deterministic
dynamics. Some theoretical and numerical results will be presented
to illustrate the efficiency of the numerical algorithms. 

\subsection{Problem and algorithm}

We first define the relation between the three functions $u(x)$,
$\lambda(x)$ and $\hat{a}(x)$ (as the special case of \eqref{eq:4-9}):
\begin{equation}
\begin{cases}
\alpha u(x)=f(x,\hat{a}(x))+\lambda(x)\cdot g(x,\hat{a}(x)),
\\
\alpha\lambda(x)=D\lambda(x)g(x,\hat{a}(x))+D_{x}^{*}g(x,\hat{a}(x))\lambda(x)+D_{x}f(x,\hat{a}(x)),
\end{cases}\label{eq:6-1}
\end{equation}
where 
\[
\hat{a}(x)\:\text{minimizes in }a\text{ of the function }\lambda(x)\cdot g(x,a)+f(x,a).
\]
We propose two iterations. 
\begin{itemize}
\item[i)] The first one is: for given functions $(\hat{a}^{k}(x),\lambda^{k}(x))$,
we define $u^{k}(x)$ as 
\begin{equation}
\alpha u^{k}(x)=f(x,\hat{a}^{k}(x))+\lambda^{k}(x)\cdot g(x,\hat{a}^{k}(x)).\label{eq:5-103}
\end{equation}
Now, we find $\lambda^{k+1}(x)$ by solving 
\begin{equation}
\alpha\lambda^{k+1}(x)-D\lambda^{k+1}(x)g(x,\hat{a}^{k}(x))=D_{x}^{*}g(x,\hat{a}^{k}(x))\lambda^{k}(x)+D_{x}f(x,\hat{a}^{k}(x)).\label{eq:5-101}
\end{equation}
We next resolve $\hat{a}^{k+1}(x)$ by minimizing 
\[
\text{the function }\lambda^{k+1}(x)\cdot g(x,a)+f(x,a)\text{ in }a,
\]
and $u^{k+1}(x)$ is constructed by 
\begin{equation}
\alpha u^{k+1}(x)=f(x,\hat{a}^{k+1}(x))+\lambda^{k+1}(x)\cdot g(x,\hat{a}^{k+1}(x)).\label{eq:5-102}
\end{equation}
\item[ii)] The second one is to describe the policy iteration as follows: given
the functions $(\hat{a}^{k}(x),u^{k}(x))$, we then set 
\begin{equation}
\lambda^{k}(x)=Du^{k}(x).\label{eq:5-104}
\end{equation}
We obtain $u^{k+1}(x)$ by solving 
\begin{equation}
\alpha u^{k+1}(x)=f(x,\hat{a}^{k}(x))+Du^{k+1}(x)\cdot g(x,\hat{a}^{k}(x)).\label{eq:5-105}
\end{equation}
We now set 
\begin{equation}
\lambda^{k+1}(x)=Du^{k+1}(x),\label{eq:5-106}
\end{equation}
and the values of the function $\hat{a}^{k+1}(x)$ can be obtained
by minimizing 
\begin{equation}
\text{the function }\lambda^{k+1}(x)\cdot g(x,a)+f(x,a)\text{ in }a.\label{eq:5-107}
\end{equation}
\end{itemize}
Since $\hat{a}^{k}(x)$ satisfies the necessary condition of optimality
\begin{equation}
D_{a}f(x,\hat{a}^{k}(x))+(D_{a}g)^{*}(x,\hat{a}^{k}(x))\lambda^{k}(x)=0,\label{eq:5-150}
\end{equation}
we can use a gradient descent method 
\begin{equation}
\hat{a}^{k+1}(x)=\hat{a}^{k}(x)-\theta^{k+1}[D_{a}f(x,\hat{a}^{k}(x))+(D_{a}g)^{*}(x,\hat{a}^{k}(x))\lambda^{k+1}(x)].\label{eq:5-151}
\end{equation}
The suitable scalar $\theta^{k+1}$ can now be obtained by a one-dimensional
optimization problem by setting 
\begin{equation}
w^{k+1}(\theta)(x)=\hat{a}^{k}(x)-\theta[D_{a}f(x,\hat{a}^{k}(x))+(D_{a}g)^{*}(x,\hat{a}^{k}(x))\lambda^{k+1}(x)],\label{eq:5-152}
\end{equation}
and 
\begin{equation}
H^{k+1}(\theta)(x)=f(x,w^{k+1}(\theta)(x))+\lambda^{k+1}(x)\cdot g(x,w^{k+1}(\theta)(x)),\label{eq:5-153}
\end{equation}
where we can now find $\theta^{k+1}$ by minimizing the function $H^{k+1}(\theta)(x)$
in $\theta$. As a result, $\theta^{k+1}$ depends on $x$ and plugging
back in \eqref{eq:5-151} to obtain $\hat{a}^{k+1}(x)$.

\subsection{Splitting up method}

\label{Sec:Iter}

As a part of both iterations \eqref{eq:5-101} and \eqref{eq:5-105}
described above we have to solve a generic linear PDE 
\begin{equation}
\alpha\lambda(x)-D\lambda(x)\cdot G(x)=F(x).\label{eq:5-108}
\end{equation}
For which we propose a parallel splitting up method\footnote{The parallel splitting up method not only reduces the original problems
into a number of separable one dimensional linear problems, but also
enables us to compute all these one dimensional linear problems by
parallel computing, for which the calibration of the fractional steps
are independent of each other \cite{LNT91}.}: knowing $\lambda^{j}(x)$, and writing $x=(x_{1},\cdots,x_{d})$,
we define $\lambda^{j+\frac{l}{d}}(x),$ $l=1,\cdots,d$, by 
\begin{equation}
\alpha\lambda^{j+\frac{l}{d}}(x)-\dfrac{\partial\lambda^{j+\frac{l}{d}}(x)}{\partial x_{l}}G_{l}(x)=Z^{j+\frac{l}{d}}(x),
\label{eq:5-109}
\end{equation}
where $Z^{j+\frac{l}{d}}(x)=\sum_{h\not=l}\dfrac{\partial\lambda^{j}(x)}{\partial x_{h}}G_{h}(x)+F(x)$.
And $\lambda^{j+1}(x)$ is defined by 
\begin{equation}
\lambda^{j+1}(x)=\frac{1}{d}\sum_{l=1}^{d}\lambda^{j+\frac{l}{d}}(x).\label{eq:5-110}
\end{equation}
Note that \eqref{eq:5-109} is a one dimensional first order differential
equation, which has an explicit solution 
\begin{equation}
\lambda^{j+\frac{l}{d}}(x)=-\int_{-\infty}^{x_{l}}\dfrac{Z^{j+\frac{l}{d}}(\xi_{l},\bar{x}_{l})}{G_{l}(\xi_{l},\bar{x}_{l})}\exp\left(\alpha\int_{\xi_{l}}^{x_{l}}\dfrac{\mathrm{d}\eta_{l}}{G_{l}(\eta_{l},\bar{x}_{l})}\right)\mathrm{d}\xi_{l}.\label{eq:5-112}
\end{equation}
Here, we have used the notation $x=(x_{l},\bar{x}_{l})$ where $\bar{x}_{l}\in\mathbb{R}^{d-1}$. 

\section{CONVERGENCE RESULTS}

\subsection{Setting of the problem }

We take 

\begin{equation}
g(x,a)=A(x)+Ba,\label{eq:7-1}
\end{equation}
such that
\[
x\mapsto A(x):\mathbb{R}^{n}\rightarrow \mathbb{R}^{n},|A(x)|\leq\gamma|x|,\:B\in\mathcal{L}(\mathbb{R}^{d};\mathbb{R}^{n}),
\]
and
\[
||DA(x_{1})-DA(x_{2})||\leq\dfrac{b|x_{1}-x_{2}|}{1+|x_{1}|+|x_{2}|}.
\]
The pay-off functional
\begin{equation}
f(x,a)=F(x)+\dfrac{1}{2}a^{*}Na,\label{eq:7-2}
\end{equation}

\[
x\mapsto F(x):\mathbb{R}^{n}\rightarrow \mathbb{R},\:|DF(x)|\leq M|x|,
\]

\[
N\in\mathcal{L}(\mathbb{R}^{d};\mathbb{R}^{d}),\text{ symmetric and invertible.}
\]
We then have 

\begin{equation}
\hat{a}(x)=-N^{-1}B^{*}\lambda(x),\label{eq:7-3}
\end{equation}
and thus the second relation (\ref{eq:6-1}) becomes 

\begin{equation}
\alpha\lambda(x)-DA^{*}(x)\lambda(x)-D\lambda(x)(A(x)-BN^{-1}B^{*}\lambda(x))=DF(x).\label{eq:7-4}
\end{equation}

\subsection{Preliminaries}

We will need conditions on $\alpha$ and $b$: $\alpha$ sufficiently
large and $b$ sufficiently small. We first assume that

\begin{equation}
\alpha-2\gamma>2\sqrt{M||BN^{-1}B^{*}||}.\label{eq:7-5}
\end{equation}
We set 

\begin{equation}
\beta=\dfrac{(\alpha-2\gamma)^{2}}{4M||BN^{-1}B^{*}||}>1.\label{eq:7-6}
\end{equation}
We define 

\begin{equation}
\varpi=\dfrac{\alpha-2\gamma-\sqrt{(\alpha-2\gamma)^{2}-4M||BN^{-1}B^{*}||}}{2||BN^{-1}B^{*}||},\label{eq:7-7}
\end{equation}
which is a solution of 

\begin{equation}
\varpi^{2}||BN^{-1}B^{*}||-(\alpha-2\gamma)\varpi+M=0.\label{eq:7-8}
\end{equation}
We next need to solve for the equation 

\begin{equation}
\nu^{2}||BN^{-1}B^{*}||-(\alpha-2\gamma)\nu+(M+b\varpi)=0,\;\nu>\varpi.\label{eq:7-9}
\end{equation}
We need that

\begin{equation}
b<\sqrt{\dfrac{||BN^{-1}B^{*}||}{M}}\:\dfrac{\beta-1}{\sqrt{\beta}-\sqrt{\beta-1}}.\label{eq:7-10}
\end{equation}
We then define 

\begin{equation}
\nu=\dfrac{\alpha-2\gamma-\sqrt{(\alpha-2\gamma)^{2}-4(M+b\varpi)||BN^{-1}B^{*}||}}{2||BN^{-1}B^{*}||}.\label{eq:7-11}
\end{equation}

\subsection{Main Result}

We can state the 
\begin{thm}
\label{theo7-1} We assume (\ref{eq:7-1}), (\ref{eq:7-2}), (\ref{eq:7-5}) and (\ref{eq:7-10}).
Then equation (\ref{eq:7-4}) has a unique solution such that 

\begin{equation}
|\lambda(x)|\leq\varpi|x|,\:||D\lambda(x)||\leq\nu.\label{eq:7-12}
\end{equation}
\end{thm}

\begin{proof}
We will use a contraction mapping argument. Let $\lambda(x)$ be a
vector of functions satisfying (\ref{eq:7-12}). We shall define a
function $\Gamma(x)$ as follows. We consider the differential equation 

\begin{equation}
\left\{\begin{aligned}
\dfrac{dy}{ds}&=A(y)-BN^{-1}B^{*}\lambda(y),\\
y(0)&=x.\end{aligned}\right.\label{eq:7-13}
\end{equation}

Since $A(x)$ and $\lambda(x)$ are uniformly Lipschitz, this equation
has a unique solution. We then define $\Gamma(x)$ by the formula 

\begin{equation}
\Gamma(x)=\int_{0}^{+\infty}\exp(-\alpha s)\left(DF(y(s))+DA^{*}(y(s))\lambda(y(s))\right)ds.\label{eq:7-14}
\end{equation}
This integral is well-defined. Indeed, from (\ref{eq:7-13}), the
second assumption (\ref{eq:7-1}) and the first property (\ref{eq:7-12}), we can assert that 

\begin{equation}
|y(s)|\leq|x|\exp(\gamma+||BN^{-1}B^{*}||\varpi)s\label{eq:7-15}
\end{equation}
and, from (\ref{eq:7-14}) we get 

\begin{align}
|\Gamma(x)|&\leq(M+\varpi\gamma)\int_{0}^{+\infty}\exp(-\alpha s)\,|y(s)|ds\notag\\
&\leq(M+\varpi\gamma)|x|\int_{0}^{+\infty}\exp(-(\alpha-\gamma-||BN^{-1}B^{*}||\varpi)s)ds\notag\\
&=\dfrac{(M+\varpi\gamma)|x|}{\alpha-\gamma-||BN^{-1}B^{*}||\varpi}=\varpi|x|,\label{eq:7-16}
\end{align}
 from the definition of $\varpi$ of (\ref{eq:7-7}) and (\ref{eq:7-8}). In particular $\Gamma(x)$
satisfies the first property (\ref{eq:7-12}). We next differentiate
in $x$ the formula (\ref{eq:7-14}). We set $Y(s)=D_{x}y(s).$ From
Equation (\ref{eq:7-13}) we obtain 

\begin{equation}
\left\{\begin{aligned}
\dfrac{dY(s)}{ds}&=DA(y(s))Y(s)-BN^{-1}B^{*}D\lambda(y(s))Y(s),\\
Y(0)&=I,\end{aligned}\right.\label{eq:7-17}
\end{equation}

then,

\begin{equation}
D\Gamma(x)=\int_{0}^{+\infty}\exp(-\alpha s)\left(D^{2}F(y(s))+DA(y(s))D\lambda(y(s))+D^{2}A(y(s)\lambda(y(s))\right)Y(s)ds.\label{eq:7-18}
\end{equation}
So 

\[
||D\Gamma(x)||\leq(M+\gamma\nu+b\varpi)\int_{0}^{+\infty}\exp(-\alpha s)||Y(s)||ds,
\]
and from (\ref{eq:7-17}) it follows that

\begin{align*}
||D\Gamma(x)||&\leq(M+\gamma\nu+b\varpi)\int_{0}^{+\infty}\exp(-(\alpha-\gamma-||BN^{-1}B^{*}||\nu)s)\,ds\\
&=\dfrac{M+\gamma\nu+b\varpi}{\alpha-\gamma-||BN^{-1}B^{*}||\nu}=\nu,
\end{align*}
 and thus $\Gamma(x)$ satisfies the second condition (\ref{eq:7-12}). 

We consider the Banach space of functions $\lambda(x):\mathbb{R}^{n}\rightarrow \mathbb{R}^{n},$ with
the norm 

\[
||\lambda||=\sup_{x}\dfrac{|\lambda(x)|}{|x|},
\]
 and the closed subset 

\[
\mathcal{C}=\{\lambda(\cdot)|\:||\lambda||\leq\varpi,\:||D\lambda(x)||\leq\nu,\forall x\}.
\]
We consider the map $\mathcal{T}:$ $\lambda\rightarrow\Gamma,$ defined
by the formula (\ref{eq:7-14}). We want to show that it is a contraction
from $\mathcal{C}$ to $\mathcal{C}.$ We pick two functions $\lambda_{1},\lambda_{2}$
in $C$, let $y_{1}(s),y_{2}(s)$ be defined by (\ref{eq:7-13})
with $\lambda=\lambda_{1},\lambda_{2}$ respectively, and $\Gamma_{1}=\mathcal{T}\lambda_{1},\Gamma_{2}=\mathcal{T}\lambda_{2}.$
We have 

\[
\dfrac{d}{ds}(y_{1}-y_{2})=A(y_{1})-A(y_{2})-BN^{-1}B^{*}(\lambda_{1}(y_{1})-\lambda_{2}(y_{2})).
\]
 Noting that 

\begin{align*}
|\lambda_{1}(y_{1})-\lambda_{2}(y_{2})|&\leq|\lambda_{1}(y_{1})-\lambda_{2}(y_{1})|+|\lambda_{2}(y_{1})-\lambda_{2}(y_{2})|\\
&\leq||\lambda_{1}-\lambda_{2}||\,|y_{1}|+\nu|y_{1}-y_{2}|,
\end{align*}
 we get, using the estimate (\ref{eq:7-15}), 

\[\left\{\begin{aligned}
\dfrac{d}{ds}|y_{1}-y_{2}|&\leq(\gamma+\nu||BN^{-1}B^{*}||)|y_{1}-y_{2}|\\
&\quad+||BN^{-1}B^{*}||\:||\lambda_{1}-\lambda_{2}||\,|x|\exp(\gamma+\varpi||BN^{-1}B^{*}||)s,\\
(y_{1}-y_{2})(0)&=x,\end{aligned}\right.
\]
 therefore 

\[
|y_{1}(s)-y_{2}(s)|\leq||BN^{-1}B^{*}||\,||\lambda_{1}-\lambda_{2}||\,|x|\exp((\gamma+\nu||BN^{-1}B^{*}||)s)\int_{0}^{s}\exp(-(\nu-\varpi)||BN^{-1}B^{*}||\tau)\,d\tau.
\]
 Finally, we obtain that

\begin{equation}
|y_{1}(s)-y_{2}(s)|\leq\dfrac{||\lambda_{1}-\lambda_{2}||\,|x|}{\nu-\varpi}(\exp((\gamma+\nu||BN^{-1}B^{*}||)s)-\exp((\gamma+\varpi||BN^{-1}B^{*}||)s)).\label{eq:7-19}
\end{equation}
 Next, from the definition of $\Gamma(x),$ we get 

\begin{align*}
\Gamma_{1}(x)-\Gamma_{2}(x)&=\int_{0}^{+\infty}\exp(-\alpha s)[DF(y_{1}(s))-DF(y_{2}(s))\\
&\quad\quad+DA^{*}(y_{1}(s))\lambda_{1}(y_{1}(s))-DA^{*}(y_{2}(s))\lambda_{2}(y_{2}(s))]ds.
\end{align*}
Writing 

\begin{align*}
&DA^{*}(y_{1}(s))\lambda_{1}(y_{1}(s))-DA^{*}(y_{2}(s))\lambda_{2}(y_{2}(s))\\
=\ &\left(DA^{*}(y_{1}(s))-DA^{*}(y_{2}(s))\right)\lambda_{1}(y_{1}(s))+DA^{*}(y_{1}(s))(\lambda_{1}(y_{1}(s))-\lambda_{2}(y_{2}(s))).
\end{align*}
 From the third line of assumption (\ref{eq:7-1}) we obtain 

\[
|\left(DA^{*}(y_{1}(s))-DA^{*}(y_{2}(s))\right)\lambda_{1}(y_{1}(s))|\leq b\varpi|y_{1}(s)-y_{2}(s)|.
\]
 Moreover 

\[
|\lambda_{1}(y_{1}(s))-\lambda_{2}(y_{2}(s))|\leq\nu|y_{1}(s)-y_{2}(s)|+||\lambda_{1}-\lambda_{2}||\,|x|\exp((\gamma+\varpi||BN^{-1}B^{*}||)s).
\]
 Collecting results, we can write 

\begin{align*}
|\Gamma_{1}(x)-\Gamma_{2}(x)|&\leq(M+b\varpi+\gamma\nu)\int_{0}^{+\infty}\exp(-\alpha s)\:|y_{1}(s)-y_{2}(s)|ds\\
&\quad+\dfrac{\gamma||\lambda_{1}-\lambda_{2}||\,|x|}{\alpha-\gamma-||BN^{-1}B^{*}||\varpi}.
\end{align*}
 Making use of (\ref{eq:7-19}), it follows that

\begin{align*}
|\Gamma_{1}(x)-\Gamma_{2}(x)|&\leq\dfrac{\gamma||\lambda_{1}-\lambda_{2}||\,|x|}{\alpha-\gamma-||BN^{-1}B^{*}||\varpi}+(M+b\varpi+\gamma\nu)\\
&\leq\dfrac{||\lambda_{1}-\lambda_{2}||\,|x|}{\nu-\varpi}\left[\dfrac{1}{\alpha-\gamma-||BN^{-1}B^{*}||\nu}-\dfrac{1}{\alpha-\gamma-||BN^{-1}B^{*}||\varpi}\right].
\end{align*}
Rearranging and using the definition of $\nu,$ see (\ref{eq:7-9}),
we finally obtain that

\begin{equation}
||\Gamma_{1}-\Gamma_{2}||\leq\dfrac{\gamma+||BN^{-1}B^{*}||\nu}{\alpha-\gamma-||BN^{-1}B^{*}||\varpi}||\lambda_{1}-\lambda_{2}||.\label{eq:7-20}
\end{equation}
We need to check that 

\begin{equation}
\dfrac{\gamma+||BN^{-1}B^{*}||\nu}{\alpha-\gamma-||BN^{-1}B^{*}||\varpi}<1\label{eq:7-21}
\end{equation}
 which is equivalent to 

\[
\alpha-2\gamma-||BN^{-1}B^{*}||(\varpi+\nu)>0,
\]
 which is true, from the definition of $\varpi$ and $\nu,$ see (\ref{eq:7-7})
and (\ref{eq:7-11}). If we call $\lambda(x)$ the unique fixed point
of $\mathcal{T}$, it satisfies (\ref{eq:7-12}) and the equation

\begin{equation}
\lambda(x)=\int_{0}^{+\infty}\exp-\alpha s\left(DF(y(s))+DA^{*}(y(s))\lambda(y(s))\right)ds.\label{eq:7-22}
\end{equation}
It is standard to check that (\ref{eq:7-12}) and (\ref{eq:7-22})
is equivalent to (\ref{eq:7-12}) and (\ref{eq:7-4}). This concludes
the proof of the Theorem.
\end{proof}

\subsection{Algorithm }

We can write the algorithm (\ref{eq:5-101}) which leads to 

\begin{equation}
\alpha\lambda^{k+1}(x)-D\lambda^{k+1}(x)(A(x)-BN^{-1}B^{*}\lambda^{k}(x))=DF(x)+DA^{*}(x)\lambda^{k}(x).\label{eq:7-23}
\end{equation}
From the contraction property obtained in Theorem \ref{theo7-1}, we
can obtain immediately 
\begin{cor}
\label{cor7-1}Under the assumptions of Theorem \ref{theo7-1}, if
we start the iteration with $\lambda_{0}$ such that $|\lambda_{0}(x)|\leq\varpi|x|$
and $||D\lambda_{0}(x)||\leq\nu,$ we have 

\begin{equation}
||\lambda^{k}-\lambda||\rightarrow0,\label{eq:7-24}
\end{equation}
 where $\lambda$ is the solution of (\ref{eq:7-4}).
\end{cor}

\subsection{Linear Quadratic case }

We take $A(x)=Ax,\:F(x)=\dfrac{1}{2}x^{*}Mx,$ then equation (\ref{eq:7-4})
becomes 

\begin{equation}
\alpha\lambda(x)=Mx+A^{*}\lambda(x)+D\lambda(x)(Ax-BN^{-1}B^{*}\lambda(x)),\label{eq:7-25}
\end{equation}
and its solution is $\lambda(x)=Px,$ with $P$ the solution of the Riccati
equation 

\begin{equation}
\alpha P=M+A^{*}P+PA-PBN^{-1}B^{*}.\label{eq:7-26}
\end{equation}
We have $\gamma=||A||$ and $b=0.$ Assumption (\ref{eq:7-5}) becomes 

\begin{equation}
\alpha>2||A||+2\sqrt{M||BN^{-1}B^{*}||}.\label{eq:7-27}
\end{equation}
We have 

\begin{equation}
\varpi=\nu=\dfrac{\alpha-2||A|-\sqrt{(\alpha-2||A|)^{2}-4M||BN^{-1}B^{*}||}}{2||BN^{-1}B^{*}||}.\label{eq:7-28}
\end{equation}
The iteration (\ref{eq:7-23}) becomes $\lambda^{k}(x)=P^{k}x$, with 

\begin{equation}
P^{k+1}(\alpha I-A+BN^{-1}B^{*}P^{k})=M+A^{*}P^{k},\label{eq:7-29}
\end{equation}
and if $||P^{0}||\leq\varpi,$ we obtain $||P^{k}-P||\rightarrow0,$ as
$k\rightarrow+\infty.$

\section{NUMERICAL RESULTS}
We now present numerical tests for the algorithm.
We consider the following values for $m,n$: $m = 10, n = 30;$ the matrices  $M,N,A,B$ are chosen arbitrarily and their values are not displayed here. 

In Figure 1,  we choose $\alpha =1$ and we pick 4 samples of the initial guess $P^{(0)}$. These choices do not satisfy the two conditions \eqref{eq:7-5} and \eqref{eq:7-6}.  Using the results of our Python code, we display the difference between $\|P^{(5)}-P^{(6)}\|,$  $\|P^{(10)}-P^{(11)}\|,$ $\|P^{(15)}-P^{(16)}\|,$ $\|P^{(20)}-P^{(21)}\|$, $\|P^{(25)}-P^{(26)}\|$. We can see that as the number of iterations $k$  increases, the difference between $\|P^{(k)}-P^{(k+1)}\|$ does not become small. This shows that the algorithm does not converge.

In Figure 2,  we choose $\alpha=1770.3688$, $0<\|P^{(0)} \|< 6.8153$ for the 4 samples of  the initial guess $P^{(0)}$. These choices satisfy the two conditions \eqref{eq:7-5} and \eqref{eq:7-6}. We can see that as the number of iterations $k$   increases, the difference between $\|P^{(k)}-P^{(k+1)}\|$  become small very small. And after $15$ iterations, these differences are essentially $0$.

In Figure 3,  we choose $\alpha=225$. In this test, the condition \eqref{eq:7-5} corresponds to $\alpha>270$.  Using the results of our Python code, we display   the difference between $\|P^{(k)}-P^{(k+1)}\|$. We can see that the difference does not converge to $0$  even after $10000$ iterations. Therefore, the condition \eqref{eq:7-5} is quite good.

{In Figure 4, we take the first choice of $P^{(0)}$ arbitrarily, and varies the values of $\alpha$ to be $250,300,400,500$. Using the results of our Python code, we plot the values of $\|P^{(5)}-P^{(6)}\|$, $\|P^{(10)}-P^{(11)}\|,$ $\|P^{(15)}-P^{(16)}\|$, $\|P^{(20)}-P^{(21)}\|$, $\|P^{(25)}-P^{(26)}\|$, $\|P^{(30)}-P^{(31)}\|$ for each value of $\alpha$ as a curve. We can see that the algorithm starts to converge very fast if $\alpha$ is big: the curves for $\alpha=250,300,400,500$ (red, green, orange and blue) are almost the $0$-line. }

\begin{figure}[h]
\begin{center}
\includegraphics[width=0.75\textwidth]{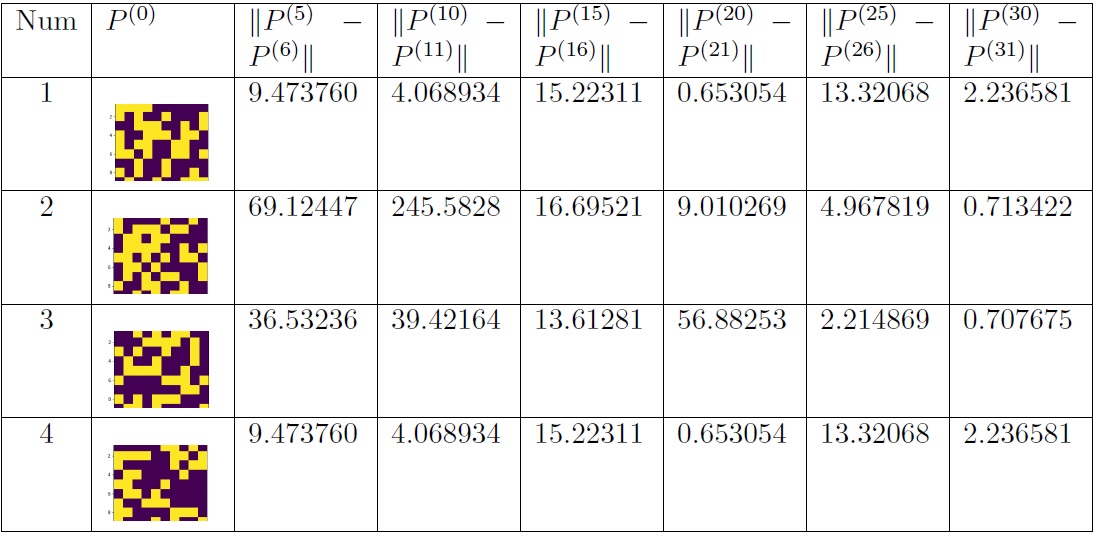}
\end{center}
\mbox{Figure 1: Solving for $P$: 4 tests  where $\alpha=1$.}
\label{Fig3}\end{figure}

\begin{figure}[h]
\begin{center}
\includegraphics[width=0.75\textwidth]{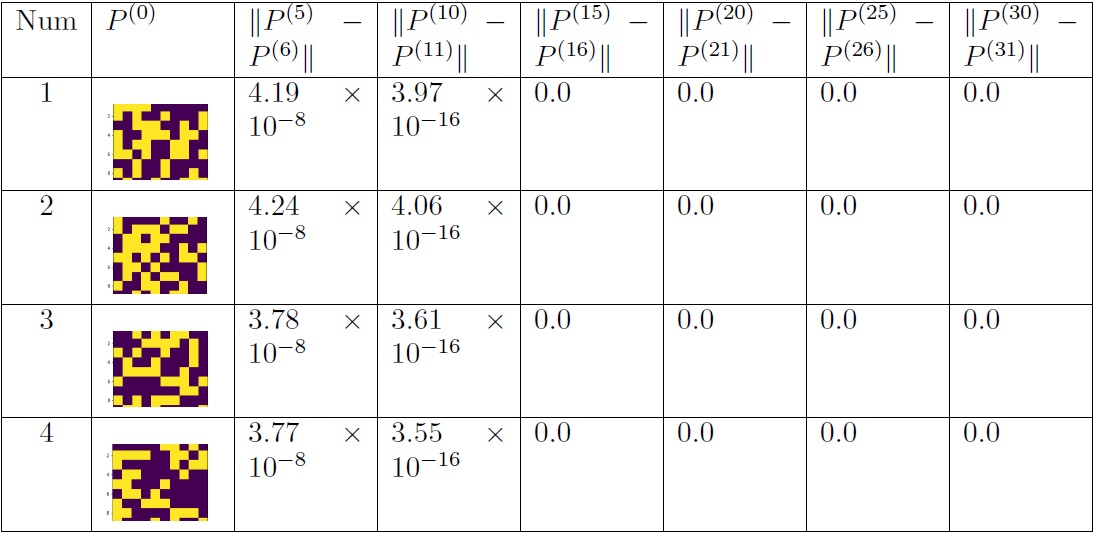}
\end{center}
\mbox{Figure 2: Solving for $P$: 4 tests  where $\alpha=1770.3688$}
\label{Fig3}\end{figure}

\begin{figure}[h]
\begin{center}
\includegraphics[width=0.99\textwidth]{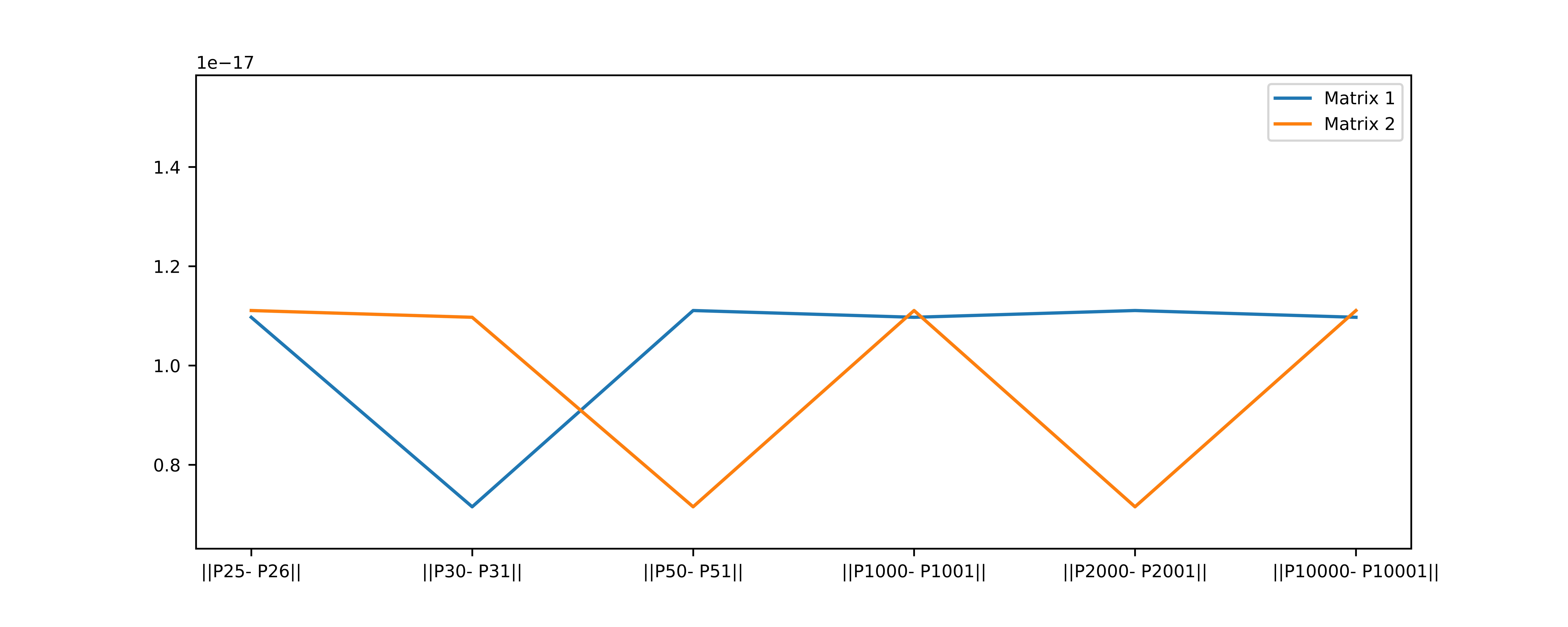}
\end{center}
\mbox{Figure 3: Solving for $P$: 4 tests  where $\alpha=225.$}
\label{Fig3}\end{figure}

\begin{figure}[h]
\begin{center}
\includegraphics[width=0.65\textwidth]{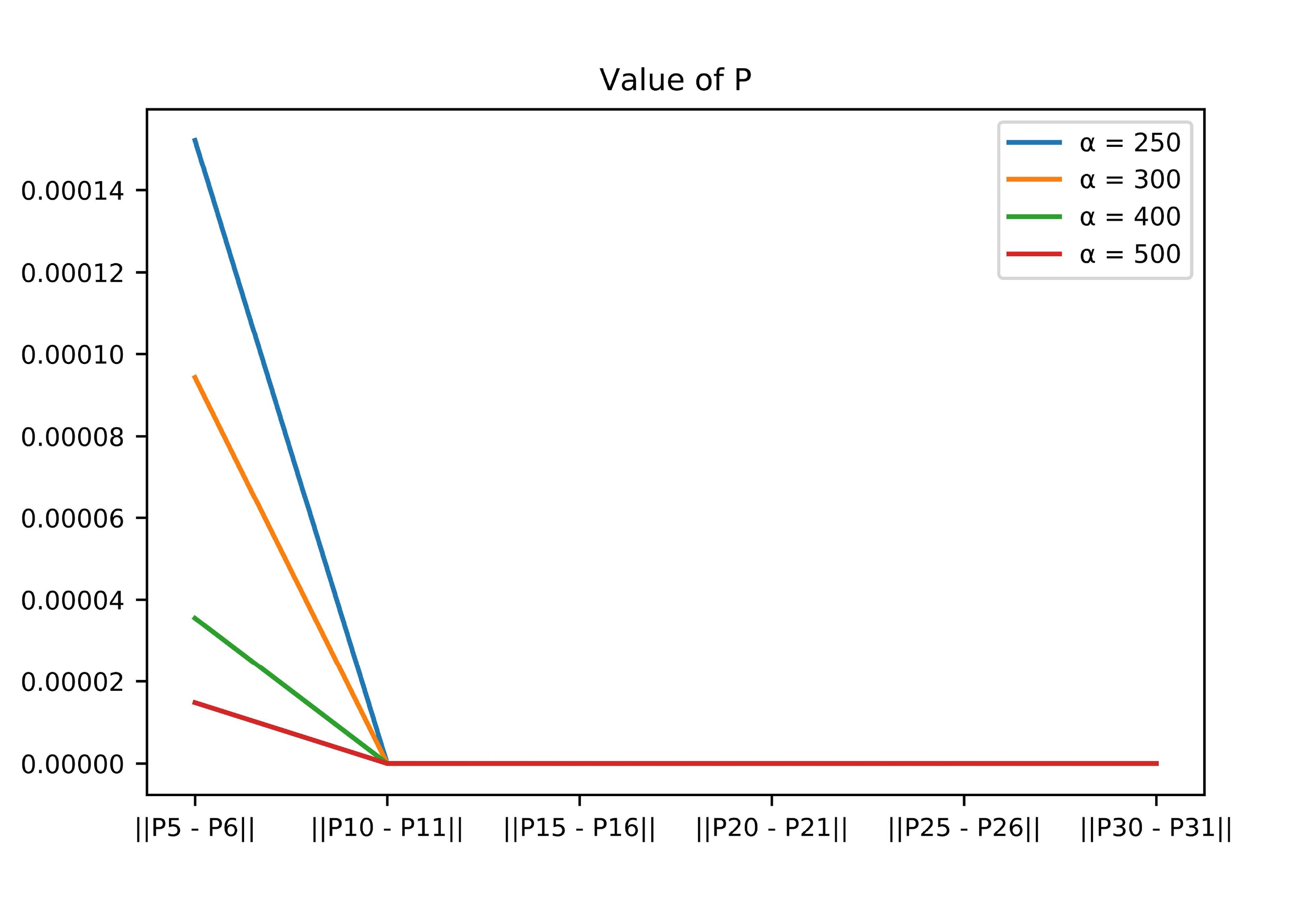}
\end{center}
\mbox{Figure 4: Solving for $P$: The convergence rate for different values of $\alpha$.}
\label{Fig3}\end{figure}

\par
\section*{Acknowledgement}
Phillip Yam acknowledges the financial supports from HKGRF-14300717 with the project title ``\textit{New kinds of Forward-backward Stochastic Systems with Applications}'', HKGRF-14300319 with the project title ``\textit{Shape-constrained Inference: Testing for Monotonicity}'', and Direct Grant for Research 2014/15 (Project No.\ 4053141) offered by CUHK. 
Xiang Zhou acknowledges the support of Hong Kong RGC GRF grants 11337216 and 11305318. Minh-Binh Tran is partially supported by NSF Grant DMS-1854453, SMU URC Grant 2020, SMU DCII Research Cluster Grant, Dedman College Linking Fellowship, Alexander von Humboldt Fellowship. Dinh Phan Cao Nguyen and Minh-Binh Tran would like to thank Prof. T. Hagstrom and Prof. A. Aceves for the computational resources.
\par

\bibliographystyle{plain}
\bibliography{MLCT}

\end{document}